\documentclass{article}
\usepackage[utf8]{inputenc}
\usepackage{times}
\usepackage{amsmath}
\usepackage{amssymb}
\usepackage[final]{graphicx}
\usepackage{floatflt}
\usepackage{color}
\usepackage{cite}
\usepackage{color}
\usepackage{wrapfig,subcaption} %subfigure}
\usepackage{hyperref}
\usepackage{multirow}% http://ctan.org/pkg/multirow
\usepackage{diagbox}

\usepackage{amssymb,amsfonts,amsmath,amstext,amsthm,times}
\usepackage{pifont}
\usepackage{graphics}
\usepackage{epsfig,psfrag}
\usepackage{wrapfig}
\usepackage{times}
\usepackage{array}
\usepackage{epstopdf}
\usepackage{enumitem,kantlipsum}
\usepackage{float}
\usepackage{aliascnt}
\usepackage{caption}
\usepackage{bbm}
\usepackage{bm}

\newaliascnt{eqfloat}{equation}
\newfloat{eqfloat}{h}{eqflts}
\floatname{eqfloat}{Equation}
\newcommand*{\ORGeqfloat}{}
\let\ORGeqfloat\eqfloat
\def\eqfloat{%
  \let\ORIGINALcaption\caption
  \def\caption{%
    \addtocounter{equation}{-1}%
    \ORIGINALcaption
  }%
  \ORGeqfloat
}

\newcommand{\ZN}{{\mathbb{Z}_{n+1}\setminus\{0\}}}
\DeclareMathAlphabet{\mathbfsl}{OT1}{ppl}{b}{it} %{OT1}{cmr}{bx}{it}

\newcommand{\alf}{\mathrm{\bm{\alpha}}}

\theoremstyle{definition}

\def\QEDclosed{\mbox{\rule[0pt]{1.3ex}{1.3ex}}} % for a filled box
\def\QED{\QEDclosed} % default to closed

\newcommand{\be}[1]{\begin{equation}\label{#1}}
\newcommand{\ee}{\end{equation}}

\renewcommand{\leq}{\leqslant}
 
\renewcommand{\geq}{\geqslant}

\newcommand{\Cref}[1]{Co\-ro\-lla\-ry\,\ref{#1}}

\DeclareMathAlphabet{\mathcal}{OMS}{cmsy}{m}{n}

\newcommand{\one}{{\mathbbm 1}}

\def\x{\textbf{x}}
\def\v{\textbf{v}}

\def\a{\textbf{a}}
\def\g{\textbf{g}}
\def\h{\textbf{h}}

\newcommand{\stexp}{\mbox{$\mathbb{E}$}}   
\newcommand{\beq}{\begin{equation}}
\newcommand{\eeq}{\end{equation}}   
\newcommand{\bea}{\begin{eqnarray}}
\newcommand{\eea}{\end{eqnarray}}

\newtheorem{theorem}{Theorem}

\newtheorem{corollary}{Corollary}[theorem]

\title{Competitive Machine Learning: Best Theoretical Prediction vs Optimization}
\author{Amin Khajehnejad\footnote{3Red Trading Group/University of Michigan Ann Arbor.} and Shima Hajimirza\footnote{Texas A\&M University.}}
\date{}
\begin{document}
\maketitle

\begin{abstract}
Machine learning is often used in competitive scenarios: Participants learn and fit static models, and those models compete in a shared platform. The common assumption is that in order to win a competition one has to have the best predictive model, i.e., the model with the smallest out-sample error. Is that necessarily true?  Does the best theoretical predictive model for a target always yield the best reward in a competition? If not, can one take the best model and purposefully change it into a theoretically inferior model which in practice results in a higher competitive edge? How does that modification look like? And finally, if all participants modify their prediction models towards the  best practical performance, who benefits the most? players with  inferior models, or those with theoretical superiority? The main theme of this paper is to raise these important questions and propose a theoretical model to answer them. We consider a study case where two linear predictive models compete over a shared target. The model with the closest estimate gets the whole reward, which is equal to the absolute value of the target. We characterize the reward function of each model, and using a basic game theoretic approach, demonstrate that the inferior competitor can significantly improve his performance by choosing optimal model coefficients that are different from the best theoretical prediction. This is a preliminary study that emphasizes the fact that in many applications where  predictive machine learning is at the service of competition, much can be gained from practical (back-testing) optimization of the model compared to static prediction improvement. 
\end{abstract}

\section{Introduction}
\noindent Machine learning is often used in competitive scenarios: Participants learn and fit static models, and those models compete in a shared platform. The common assumption is that in order to win a competition one has to have the best learning model, i.e., the model with the smallest out-sample error. Is that necessarily true? The main theme of this paper is to raise this important question and propose a theoretical model to answer it. \\

\noindent In competitive machine learning, every player has a predictive model for the target, and the closest, the fastest or the most reliable model wins. What distinguishes this from static learning problems is that in a competition, target and rewards are shared among all players, and the distribution of rewards is not necessarily the same as that of  competitors'  model merits. For instance, the rules may dictate that absolute winners take all or most of the rewards, while everyone else gets nothing \cite{noe2005winner,dixit2015games}. In other words, a learning model that is really good but not quite the best can earn as much as a very poor model. The best example of this is in a one-time bidding game: The closest bid to the fair value of an asset wins the auction, and the reward is the underlying transaction \cite{kagel1993independent} (which may or may not be an appealing reward after all). The merit of a prediction can be assessed in the time domain as opposed to the space of target unit, or some combination of both. This is for example the case in High-Frequency electronic trading: The trader who has the fastest realization of a fair quote, and consequently the shortest reaction time, has the highest fill ratio, and is able to secure a profit above the transaction cost\cite{kearns2013machine}. Meanwhile the slow trader's fate is doomed at missed fills or adverse selection \cite{foucault1999order,aldridge2009high}. The extent of competitive machine learning is vast. Bidding and auction models are used in various on-line platforms where big data and AI play important roles, including e-commerce, travel reservation, on-line advertising and many more \cite{baluja2006browsing,schafer2001commerce}. Many marketing businesses revolve around competing over a shared pool of targeted customers' attentions (reward) by designing smarter recommendation algorithms (prediction models) \cite{ngai2009application,kim2005customer}. Ride sharing rivals use machine learning algorithms to predict surge times and provide recommendations or incentives to drivers for more rides and lower pickup times\cite{kamar2009collaboration}. Machine learning is now an essential part of  air traffic control, which is another example of an interactive platform of rewards (faster routs) and losses (delays) \cite{hansen2004genetic}. In fact, with the exception of  applications in natural sciences and bio-informatics, there is an element of competition in almost all practical machine learning problems.\\  

\noindent \textbf{Key  Questions}. Noting the essence of competitive machine learning, we now ask the previous question again. Does the best theoretical predictive model for a target always yield the best reward in a competition? If not, can one take the best model and purposefully change it into a theoretically inferior model which in practice results in a higher competitive edge? How does that modification look like? And finally, if all participants gear their prediction models towards the  best practical performance, who benefits the most? players with  inferior models, or those with theoretical superiority? These are some very important questions that have never been formally stated and studied to the best of the authors' knowledge. \\

\noindent \textbf{Our Contributions}. We provide a theoretical framework to  answer the above raised questions. We consider a simple case where two participants, A and B, compete over a shared target. Each participant has a prediction model for the target value, and the player with the closest prediction collects the entire reward, which is equal to the absolute value of the target.  The inferior player earns nothing. The predictions are based on underlying machine learning models that each player has independently built and trained. To make the study feasible, we focus on a simple case of linear prediction models: The target resource is the linear sum of $n$ underlying independent factors (features). Each player has access to a limited subset of those factors ($S_1,S_2$), and unlimited training data. We also make the assumption of ``linear knowledge'' regime, which means ${|S_1|}/{n}$,${|S_2|}/{n}$ and ${|S_1\cap S_2|}/{n}$ are all constants. Each player can obtain his best theoretical prediction for the target. We study this problem in a game theoretic framework. A ``strategy'' for each player is defined as a linear prediction model using the available features with a particular set of coefficients. We show that this is a constant-sum game where the expected reward of every player can be fully characterized and computed for any pair of selected strategies. We focus on a class of  strategies where the models' coefficients are ``symmetric'' with respect to both features sets. We show that for symmetric pure strategies, play A's reward is $R_1=c n$ where $c>0$ is a computable constant. When both players use their theoretical models,  $R_1=c_1 n$ for some other $c_1>0$. We further demonstrate that each player can secure a max-min reward over all symmetric strategies. This allows us to prove that the guaranteed reward for player A is $c^*n$ for some other constant $c^*$. We numerically demonstrate that in most cases $c^*/c_1 > 1$, if A is the inferior player (i.e. $|S_1| < |S_2|$). In other words, the inferior player can gain additional reward from practical (e.g. back-testing or real-world) optimization of his model coefficients,  and $c^*/c_1$ can be as high as 1.8. This is a fascinating counter intuitive result. It basically means that the inferior player can purposefully tweak his theoretical prediction model into another model that performs up to 1.8 times better in competitive learning, despite having a higher mean-square prediction error. Additionally, by looking into the chosen coefficients, we observe that the model tweaking (including both shared and unique features) always consists of magnifying the coefficients, though we cannot mathematically prove this latter point at present. \\

\section{Competitive Linear Models}

Suppose there are two players A and B who are competing to predict a target value $y$ and collect rewards as a result of their predictions. We model $y$ as the linear combination of $n$ features $x_1,...,x_n$:
\bea
y &=& \sum_{1\leq i\leq n}x_i  \nonumber \\
\x_i&\sim&\mathcal{N}(0,1)
\label{eq:main}
\eea  

\noindent where $\x_i$s are independent from  each other. The players each make an estimation about $y$ by using a subset of the features. In other words, A has access to a subset of features indexed in $S_1\subset \ZN$, and B has access to another subset of features indexed in $S_2\subset \ZN$. Without loss of generality, we assume that $|S_1| < |S_2|$, so player A is at a theoretical disadvantage. Let us call the estimations of the two players $y_1$ and $y_2$. The game rules dictates that whoever has the closest estimation wins all of the reward, which in this case is proportional to the absolute resource value $|y|$. Therefore, for a particular instance, the reward for the two players is as follows: 
\beq 
r_1=|y|\one\left(|e_1|<|e_2|\right),r_2=|y|\one\left(|e_1|\geq |e_2|\right),
\eeq 
\noindent where $e_1=y-y_1$ and $e_2=y-y_2$ are the prediction errors. The competition is only meaningful in a statistical sense, therefore the average rewards should be considered: 
\beq 
R_1=\stexp\{|y|\one\left(|e_1|<|e_2|\right)\}, R_2=\stexp\{|y|\one\left(|e_1|\geq|e_2|\right)\},
\eeq 

\noindent Note that  based on this definition $R_1+R_2=\stexp |y|$, therefore  we are dealing with a constant-sum game. The theoretical best estimations for $y_1$ and $y_2$ are the maximum likelihood estimations: 

\beq  
y^{(T)}_1 =  \sum_{i\in S_1} x_i,~ y^{(T)}_2 =  \sum_{i\in S_2} x_i \\
\eeq 

\noindent Assuming  players have had sufficiently large sets of training data, a linear regression will yield the above theoretical models. However, we consider the possibility that each player can choose a different prediction model. We limit the study to the set of all linear prediction models. A pair of strategies for A and B can be represented with a pair of coefficient vectors   $\alf_1=(\alpha_{11},\cdots,\alpha_{1n})$, and $\alf_2=(\alpha_{21},\cdots,\alpha_{2n})$, where the corresponding prediction models are: 
\beq
y_1 = \sum_{i\in S_1}\alpha_{1,i} x_i, ~ y_2 = \sum_{i\in S_2}\alpha_{2,i}x_i 
\eeq

% \begin{tabular}{|c||l|l|l||l|l|l|}
%   \hline
%   \multirow{2}{*}{Title} 
%       & \multicolumn{3}{c||}{Category~A} 
%           & \multicolumn{3}{|c|}{Category~B} \\             \cline{2-7}
%   & Item~1 & Item~2 & Item~3 & Item~1 & Item~2 & Item~3 \\  \hline
%   $X$ & 1 & 2 & 3 & 1 & 2 & 3 \\      \hline
%   $Y$ & 1 & 2 & 3 & 1 & 2 & 3 \\      \hline
% \end{tabular}

% \begin{tabular}{|l||*{5}{c|}}\hline
% \backslashbox{Room}{Date}
% &\makebox[3em]{5/31}&\makebox[3em]{6/1}&\makebox[3em]{6/2}
% &\makebox[3em]{6/3}&\makebox[3em]{6/4}\\\hline\hline
% Meeting Room &&&&&\\\hline
% Auditorium &&&&&\\\hline
% Seminar Room &&&&&\\\hline
% \end{tabular}

% \begin{table}
% \centering
% \captionof{table}{Your caption here}
% \begin{tabular}{|c|c|c|c|c|}\hline
% \diagbox[width=10em]{$x_1$}{$x_2$}&
% -1 & 0 & 1 \\ \hline
% -1 & $2\one(|\alpha_1 - 2| < 1)$ & $\one(|\alpha_1 - 1| < 1)$ & $0$  \\ \hline
% 0 & 0 & 0 & 0 \\ \hline
% 1 & 0 & $\one(|\alpha_1 - 1| < 1)$  & $2\one(|\alpha_1 - 2| < 1)$  \\ \hline
% \end{tabular}
% \end{table}

\newcommand{\var}{{\text{var}}}
\newcommand{\cov}{{\text{cov}}}
\newcommand{\non}{{\nonumber}}
\newcommand{\SAC}{S^c_1}
\newcommand{\SBC}{S^c_2}
\newcommand{\SABC}{S_1\cap S^c_2}
\newcommand{\SAB}{S_1\cap S_2}
\newcommand{\SACB}{S^c_1\cap S_2}
\newcommand{\SACBC}{S^c_1\cap S^c_2}

\noindent We also consider the \textit{linear knowledge} regime where the following holds for some constants $g_1,g_2,g_{12}$: 
\beq
|S_1|=g_1n,|S_2|=g_2n,|\SAB|=g_{12}n.
\eeq

\subsection{A Simple Example}
Consider the following simple example consisting of four variables: 
\begin{align} 
&y=x_1+x_2+x_3+x_4, \nonumber \\
&y_1=\alpha_1x_1+\alpha_2x_2, y_2=\beta_1x_2+\beta_2x_3 + \beta_3x_4
\label{eq:example}
\end{align} 
\noindent For further simplicity, let us also assume that $x_i$s can randomly take values in $\{0,1\}$ with equal probabilities instead of normal distribution. Player A has an inferior prediction model due to lower number of available features. It is possible to list out error values and  rewards for all 16 possible combinations of $x_i$s, as demonstrated in Table \ref{table:1}. 

\begin{table}[H]
\centering
\captionof{table}{Values of errors, target and reward for player A for different possible combinations of features values.}
\begin{tabular}{|c|c|c|c|c|}\hline
$x_1x_2x_3x_4$ & $e_1$ & $e_2$ & $y$ & $r_1$ \\ \hline \hline 
0000 & 0 & 0 & 0 & 0 \\ \hline
0001 & 1 & $1-\beta_3$ & 1 & $\one(1<|1-\beta_3|)$ \\ \hline
0010 & 1 & $1-\beta_2$ & 1 & $\one(1<|1-\beta_2|)$ \\ \hline
0011 & 2 & $2-\beta_2-\beta_3$ & 1 & $2\times\one(2<|2-\beta_2-\beta_3|)$ \\ \hline
0100 & $1-\alpha_2$ & $1-\beta_1$ & 1 & $\one(|1-\alpha_2|<|1-\beta_1|)$ \\ \hline
0101 & $2-\alpha_2$ & $2-\beta_1-\beta_3$ & 2 & $2\times\one(|2-\alpha_2|<|2-\beta_1-\beta_3|)$ \\ \hline
0110 & $2-\alpha_2$ & $2-\beta_1-\beta_2$ & 2 & $2\times\one(|2-\alpha_2|<|2-\beta_1-\beta_2|)$ \\ \hline
0111 & $3-\alpha_2$ & $3-\beta_1-\beta_2-\beta_3$ & 3 & $2\times\one(|3-\alpha_2|<|3-\beta_1-\beta_2-\beta_3|)$ \\ \hline
1000 & $1-\alpha_1$ & 1 & 1 & $\one(|1-\alpha_1| < 1)$ \\ \hline
1001 & $2-\alpha_1$ & $2-\beta_3$ & 2 & $2\times\one(|2-\alpha_1| < |2-\beta_3|)$ \\ \hline
1010 & $2-\alpha_1$ & $2-\beta_2$ & 2 & $2\times\one(|2-\alpha_1| < |2-\beta_2|)$ \\ \hline
1011 & $3-\alpha_1$ & $3-\beta_2-\beta_3$ & 3 & $3\times\one(|3-\alpha_1| < |3-\beta_2-\beta_3|)$ \\ \hline
1100 & $2-\alpha_1-\alpha_2$ & $2-\beta_1$ & 2 & $2\times\one(|2-\alpha_1-\alpha_2| < |2-\beta_1|)$ \\ \hline
1101 & $3-\alpha_1-\alpha_2$ & $3-\beta_1-\beta_3$ & 3 & $3\times\one(|3-\alpha_1-\alpha_2| < |3-\beta_1-\beta_3|)$ \\ \hline
1110 & $3-\alpha_1-\alpha_2$ & $3-\beta_1-\beta_2$ & 3 & $3\times\one(|3-\alpha_1-\alpha_2| < |3-\beta_1-\beta_2|)$ \\ \hline
1111 & $3-\alpha_1-\alpha_2$ & $4-\beta_1-\beta_2-\beta_3$ & 4 & $4\times\one(|4-\alpha_1-\alpha_2| < |4-\beta_1-\beta_2-\beta_3|)$ \\ \hline
%-1 & $2\one(|\alpha_1 - 2| < 1)$ & $\one(|\alpha_1 - 1| < 1)$ & $0$  \\ \hline
%1 & 0 & $\one(|\alpha_1 - 1| < 1)$  & $2\one(|\alpha_1 - 2| < 1)$  \\ \hline
\end{tabular}
\label{table:1}
\end{table}

\noindent From there, the expected reward of A is computable for every setting of $\alpha_i$s and $\beta_i$s, which is essentially the average of  column $r_1$ is table $\ref{table:1}$. The total expected reward $R_1+R_2$ in this case is $\stexp|y|=2$. When both players use their respective theoretical models, i.e. $\alpha_i=1~\forall 1\leq i \leq 2$ and $\beta_i=1~\forall 1\leq i \leq 3$, then $R_1=0.1875$ which is considerably lower than $R_2=1.8125$. For further simplicity, let us now assume that the players are only allowed to choose from strategies with equal coefficients, i.e. $\alpha_1=\alpha_2=\alpha$, and $\beta_1=\beta_2=\beta_3=\beta$. If player A unilaterally chooses a prediction model different than the theoretical model, then the reward profile looks as Figure \ref{fig:ex_reward} for different values of $\alpha$. The reward peak is very close to a fair split. To achieve this, A must use an exaggerated model with $1.5 < \alpha < 2$. An intuitive justification is that by choosing larger coefficients, the chance of conceding a smaller error value becomes lower for A. However, this will also increase the (conditional) chance of winning tail events where target values  are larger. 

\begin{figure}[H]
\centering
\includegraphics[width=0.6\textwidth]{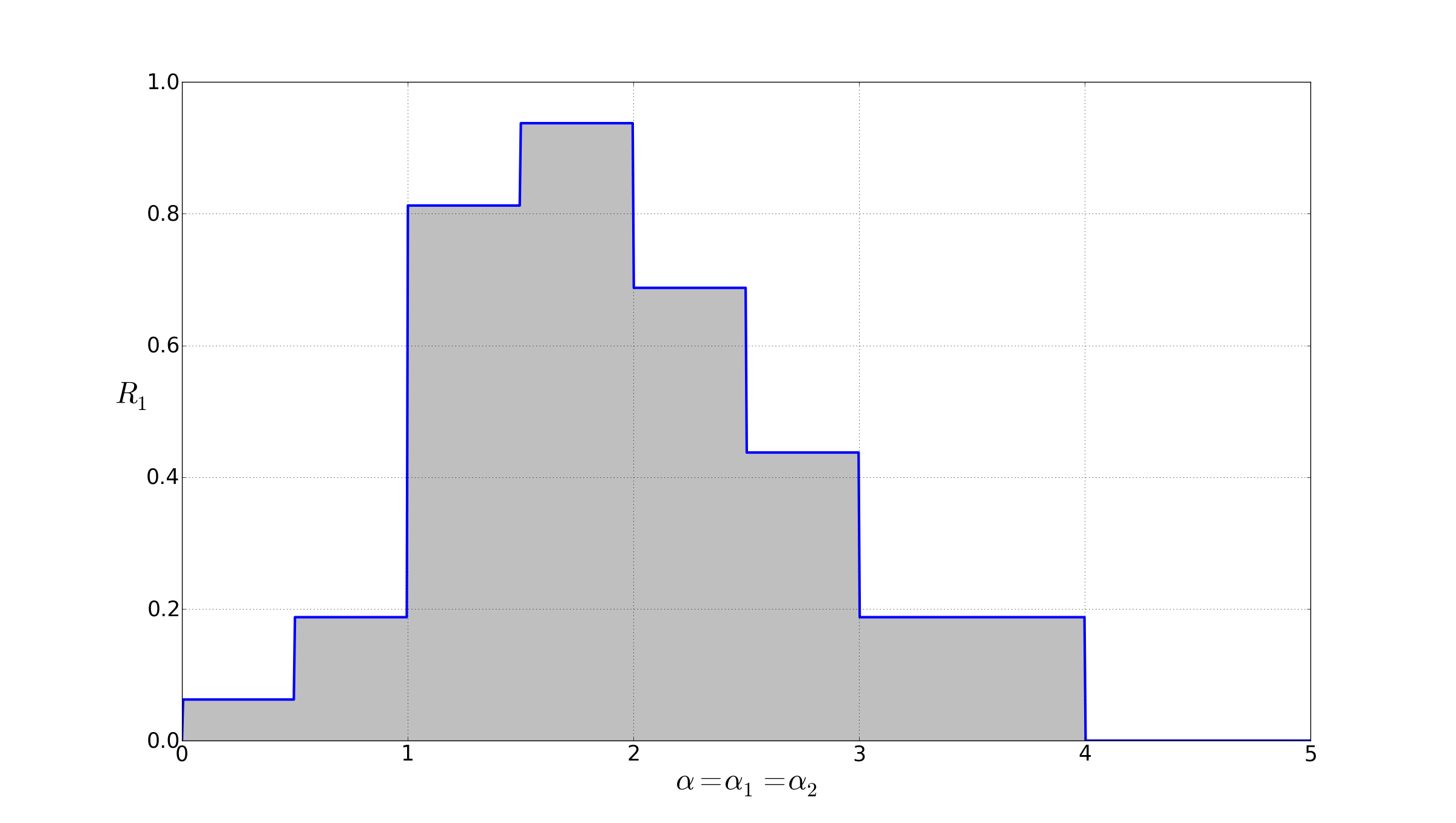}
\caption{Reward of player A if the opponent sticks to the best theoretical prediction model.}
\label{fig:ex_reward}
\end{figure}

\noindent Obviously, if B is smart he will not permit a significant loss and will try to adapt his strategy as well. From a game-theoretic perspective, there is a mixed-strategy Nash equilibrium \cite{gibbons1992primer}, which in this case is the solution of the minmax problem over all mixed strategies \cite{nisan2007algorithmic}. However, our focus here is solely on pure strategies, and from that angle, A  can at least guarantee the following maxmin value:
\beq 
\underset{\alpha}{\max}\underset{\beta}{\min}R_1(\alpha,\beta)
\eeq 

\noindent The heat map of $R_1$ as a function of $\alpha,\beta$ is illustrated in Figure \ref{fig:ex_reward_2d}. $\alpha=2.1$ leads to a guaranteed (maxmin) reward of $0.6875$ for player A, which is more than 3.5 times what he can earn with the theoretical model. 

\begin{figure}[H]
\centering
\includegraphics[width=0.6\textwidth]{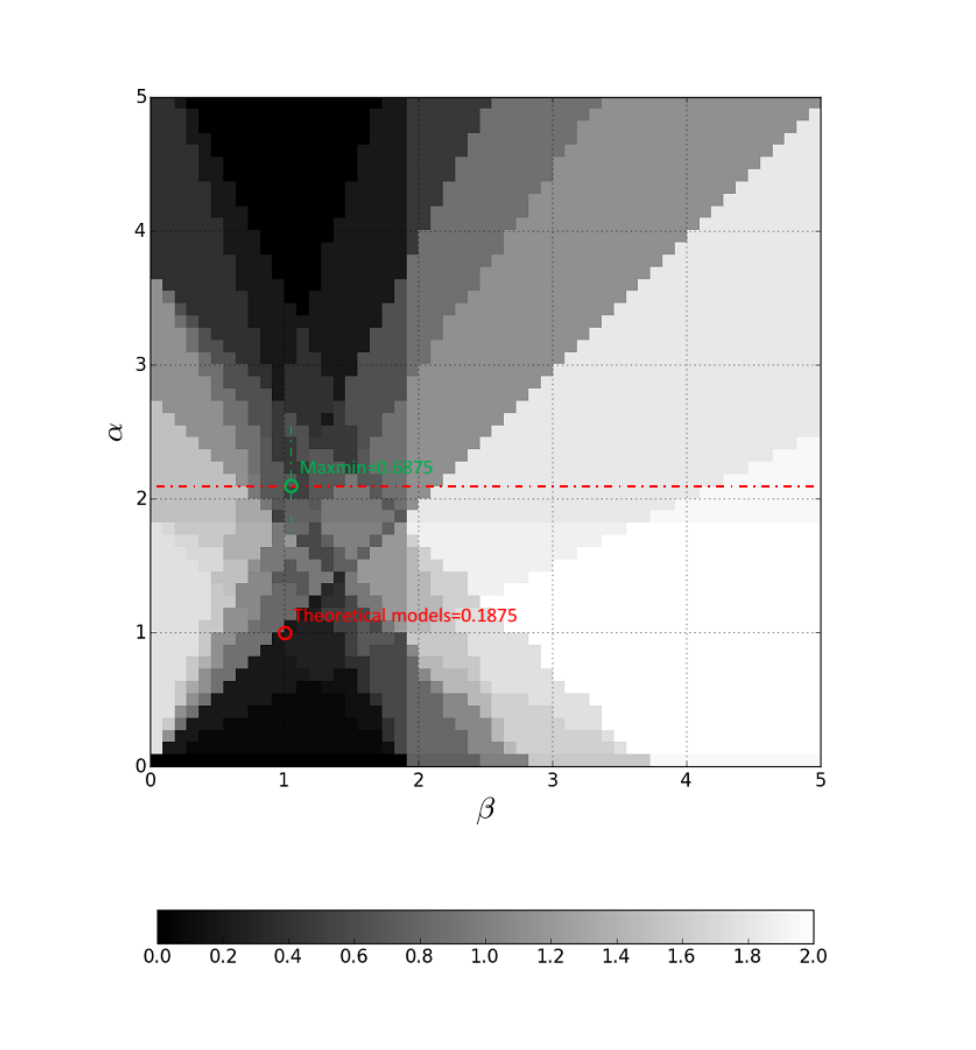}
\caption{Reward of player A as a function of coefficient values (strategies) $\alpha,\beta$.}
\label{fig:ex_reward_2d}
\end{figure}

\section{Derivations and Results}

We now turn our attention back to the generic model of (\ref{eq:main}) and study it more rigorously. Our first result states that for every strategy pair the reward is computable.  

\begin{theorem}
For every strategy pair $(\alf_1,\alf_2)$:  
\beq R_1=\int_{\x}{|x_1|\one\left(|x_2|<|x_3|\right)\Phi_{0,\Sigma}(x)d\x},\label{eq:R}\eeq 
\noindent where $\Phi_{0,\Sigma}$ is the pdf of the 3d-normal distribution with mean zero and covariance $\Sigma = 
\begin{bmatrix}
1 & v_{01} & v_{02} \\
v_{01} & v_{11} & v_{12} \\
v_{02} & v_{12} & v_{22} \\
\end{bmatrix}
$, where:

\noindent $v_{01} = |\SAC| + \underset{i\in S_1}{\sum}{(1-\alpha_{1,i})}$, 

\noindent $v_{02} = |\SBC| + \underset{i\in S_2}{\sum}{(1-\alpha_{2,i})}$,

\noindent $v_{12} = |\SAC \cap \SBC| + \underset{i\in \SAC \cap S_2}{\sum}{(1-\alpha_{2,i})} + \underset{i\in S_1 \cap \SBC}{\sum}{(1-\alpha_{1,i})} + \underset{i\in S_1 \cap S_2}{\sum}{(1-\alpha_{1,i})(1-\alpha_{2,i})}$,

\noindent $v_{11} = |\SAC| + \underset{i\in S_1}{\sum}{(1-\alpha_{1,i})^2}$,

\noindent $v_{22} = |\SBC| + \underset{i\in S_2}{\sum}{(1-\alpha_{2,i})^2}$. 

\label{theorem:R}
\end{theorem}

\noindent The integral of (\ref{eq:U}) can be computed within any arbitrary precision by using the method of sigma points based on physical Hermite polynomials $H_n(x)$ \cite{vstecha2012unscented}: 
\begin{align} 
&\int_{\x}{|x_1|\one\left(|x_2|<|x_3|\right)\Phi_{0,\Sigma}(x)d\x} = \lim_{m\rightarrow \infty} \sum_{1\leq i,j,k \leq m} W^{(i,j,k)}\cdot|T^{(i,j,k)}_1|\cdot\one\left(|T^{(i,j,k)}_2| < |T^{(i,j,k)}_3|\right) \nonumber \\
&T^{(i,j,k)} = \sqrt{\Sigma} Z^{(i,j,k)}, W^{(i,j,k)}=w_iw_jw_k, Z^{(i,j,k)}=\left(\zeta_i,\zeta_j,\zeta_k\right), 
\end{align} 
\noindent where $\zeta^{(i)}$s are the roots of $H_m(x)$, and $w_i={m!}/(n^2 H_{m-1}^2(\zeta^{(i)}))$. \\

\noindent Using the  premises of elementary game theory, one can easily deduce that player A can secure the following reward:
\beq 
R_1^*=\underset{\alf_1}{\mathrm{max}}\underset{\alf_2}{\mathrm{min}} R_1.
\eeq 
\noindent However, despite computability of the reward function, any optimization over $\alf_1,\alf_2$ involves $O(n)$ variables. Therefore any attempt to derive explicit or computable relationships for $R_1^*$ is hopeless.  To facilitate this computation, we need to make the following definition: \\

\newcommand{\Hh}{{\mathcal{H}}}

\noindent \textbf{Definition} A strategy \normalfont{$\alf_1$} for player A is called \textit{model-symmetric} (or symmetric in brief) if 
{\normalfont 
\beq
\alpha_{1,i} = \left\{\begin{array}{c} a_{11}~\forall~i\in S_1\cap S^c_2 \\ a_{12}~\forall~i\in S_1\cap S_2
\end{array}\right.,
\eeq 
}
\noindent for some real scalars $a_{11},a_{12}$. We use the notation  $\alf_1=\h(a_{11},a_{12})$, and denote the set of all model-symmetric strategies for player A with $\Hh_1$.  A strategy \normalfont{$\alf_2$} for player B is called model-symmetric if 
{\normalfont 
\beq
\alpha_{2,i} = \left\{\begin{array}{c} a_{21}~\forall~i\in S_2\cap S^c_1 \\ a_{22}~\forall~i\in S_1\cap S_2
\end{array}\right.,
\eeq 

}
\noindent for some real scalars $a_{21},a_{22}$. We use the notation  $\alf_2=\h(a_{21},a_{22})$, and denote the set of all model-symmetric strategies for player B with $\Hh_2$. We also define $\Hh=\Hh_1\times\Hh_2$.\\

% \begin{align} 
% &|S_1|=g_1n,|S_2|=g_2n,|\SAB|=g_{12}n,\\
% &|\SACB|=(g_2-g_{12})n,|\SABC|=(g_1-g_{12})n,|\SACBC|=(1-g_1-g_2+g_{12})n,
% \end{align}

\noindent The next corollary states that for model-symmetric strategies the reward is $R_1=cn$, where $c$ is independent of $n$ and only depends on parameters of the linear knowledge regime. It follows immediately from Theorem \ref{theorem:R} and the definition of symmetric strategies. 

\begin{corollary}
{\normalfont If $\alf_1=\h(a_{11},a_{12})$ and $\alf_2=\h(a_{21},a_{22})$ are model-symmetric strategies, then $R_1=nU_{\g}(\a_1,\a_2)$}, where:

{\normalfont
\beq U_{\g}(\a_1,\a_2) = \int_{\x}{|x_1|\one\left(|x_2|<|x_3|\right)\Phi_{0,\Sigma}(x)d\x} 
\label{eq:U}
\eeq 
}
% {\normalfont
%  \beq U_{\g}(\a_1,\a_2) = 2|2\pi\Sigma|^{-0.5}\int_{-x_2}^{x_2}dx_3\int_{0}^{\infty}dx_2\int_{-\infty}^{\infty}{|x_1|\exp(-\x \Sigma^{-1}\x)dx_1} 
%  \eeq 
% }

\noindent where $\Phi_{0,\Sigma}$ is the pdf of 3d-normal distribution with mean zero and covariance $\Sigma = 
\begin{bmatrix}
1 & v_{01} & v_{02} \\
v_{01} & v_{11} & v_{12} \\
v_{02} & v_{12} & v_{22} \\
\end{bmatrix}
$, and:

\noindent $v_{01} = (1-g_1) + (1-a_{11})(g_1-g_{12}) + (1-a_{12})g_{12}$, 

\noindent $v_{02} = (1-g_2) + (1-a_{21})(g_2-g_{12}) + (1-a_{22})g_{12}$,

\noindent $v_{12} = (1-g_1-g_2+g_{12}) + (1-a_{11})(g_1-g_{12}) + (1-a_{21})(g_2-g_{12}) + (1-a_{12})(1-a_{22})g_{12}$,

\noindent $v_{11} = (1-g_1) + (1-a_{11})^2(g_1-g_{12}) + (1-a_{12})^2g_{12}$,

\noindent $v_{22} = (1-g_2) + (1-a_{21})^2(g_2-g_{12}) + (1-a_{22})^2g_{12}$, 

\noindent {\normalfont $\a_1=[a_{11},a_{12}],~\a_2=[a_{21},a_{22}],~\g=[g_1,g_2,g_{12}]$.}
\label{corollary:U}
\end{corollary}

\textbf{ }\\

\begin{theorem}
$
R_1^*\geq \underset{(\alf_1,\alf_2)\in\Hh}{ \underset{\alpha_1}{\mathrm{max}}\underset{\alpha_2}{\mathrm{min}}}~R_1(\alpha_1,\alpha_2)=n \underset{\a_2\in\mathbb{R}^2}{\mathrm{max}}\underset{\a_1\in\mathbb{R}^2}{\mathrm{min}}~U_{\g}(\a_1,\a_2).
$ 
\label{theorem:symmetric}
\end{theorem}
\begin{proof}
we only need to show that for a symmetric $\alf_1$, the minimizer  $\alf_2$ of $R(\alf_1,\alf_2)$ must be symmetric too. if $i$ and $j$ are two indices in $\SACB$, then any change in the values of $\alpha_{2,i}$ and $\alpha_{2,j}$ only effect $v_{0,2}=\cov(y,e_2)$ and $v_{2,2}=\var(e_2)$ (recall the notations of model covariances in Theorem \ref{theorem:R}). Let $\alf^*_2$ be the minimizer of $R_2$ and $v^*_{0,2}$, and $v^*_{22}$ be the resulting values for $v_{0,2}$ and $v_{22}$. Define a new vector $\alf_2$ which is equal to $\alf^*_2$ in every coordinate except $i$ and $j$, where $\alpha_{2,i}=\alpha_{2,j}=(\alpha^*_{2,i} + \alpha^*_{2,j})/2$. We can conclude that:
\begin{align}
&v_{0,2}  - v^*_{0,2} = 0, \\
&v_{2,2}  - v^*_{2,2} = 2(1-(\alpha^*_{2,i} + \alpha^*_{2,j})/2)^2 - (1-\alpha^*_{2,i})^2 - (1-\alpha^*_{2,j})^2 = -(\alpha^*_{2,i}-\alpha^*_{2,j})^2/2 \leq 0.  
\end{align} 
This means that $\alf_2$ results in a model for B which has the same covariance with the target, but smaller or equal error variance, while everything else remains the same. Such a model is equal or superior for player B. Therefore $R_1$ gets smaller which contradicts the choice of $\alf^*_2$, unless $\alpha^*_{2,i}=\alpha^*_{2,i}$.  
\noindent Similarly for two indices $i,j$ in $S_1\cap S_2$, a similar transformation of $\alf^*$ results in a model with smaller or equal $v_{2,2}$ and identical $v_{0,2}$. In this case since $\alf_1$ has only one unique value, $v_{1,2}$ will remain the same as well. Thus a similar argument holds, and the values of $i$ and $j$ coordinates in $\alf^*_2$ are identical.  
\end{proof}

\textbf{ }\\

\noindent Based on theorem \ref{theorem:symmetric}, we can use a computational optimization over four variable $a_1,a_2,b_1,b_2$ to obtain a lower bound for the guaranteed rate of the inferior player $R_1^*$. For simplicity, let us define: 

\beq U^{\one}_{\g} = U_{\g}\left((1,1),(1,1)\right),~U^*_{\g} =  \underset{\a_2\in\mathbb{R}^2}{\mathrm{max}}\underset{\a_1\in\mathbb{R}^2}{\mathrm{min}}~U_{\g}(\a_1,\a_2).\eeq 

\noindent $U^{\one}_{\g}$ is the (normalized) reward of player A when both players use their theoretical prediction models, while $U^*_{\g}$ is the conservative guaranteed reward if player A chooses an optimally distorted model. The ratio $U^{*}_{\g}/U^{\one}_{\g}$ is therefore the game-theoretic notion of ``gain'' that player A can achieve in this competition by optimizing his model's coefficients. Based on Theorem \ref{theorem:symmetric} and Corollary \ref{corollary:U}, a lower bound for this gain is completely computable for every regime vector $\g$. We have computed and will present this for a special case where $g_{12}=g1g2$. This case presents a ``typical'' subsets intersection for large $n$ if $S_1$ and $S_2$ are selected independently and uniformly. The 2d plots of $U^{\one}_{\g}$ and $U^{*}_{\g}/U^{\one}_{\g}$ as functions of $g_1,g_2$ are depicted in Figures \ref{fig:a} and \ref{fig:b} respectively.  Note that the gain can be as large as 1.8 for player A. In general, the more inferior A (the smaller g1/g2), the larger the gain.  

\begin{figure}[H] \centering
    \begin{subfigure}[]{0.5\linewidth}
        \includegraphics[width=0.8\textwidth]{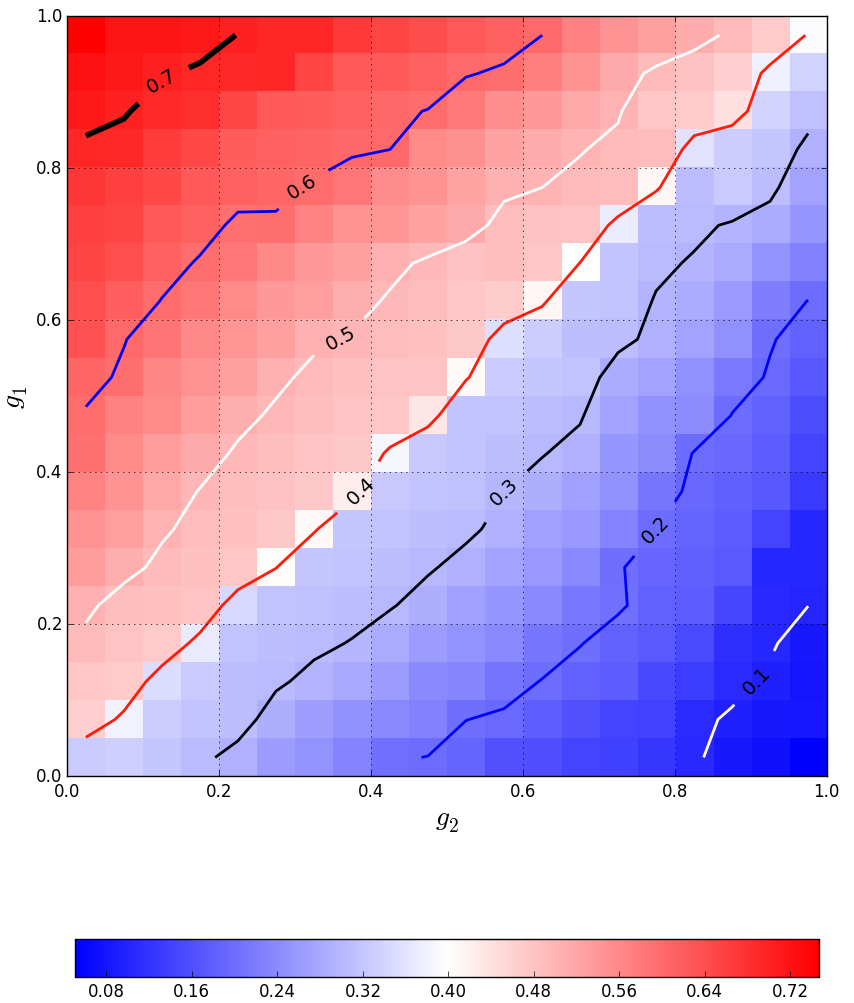}
        \caption{$U_{\g}^{\one}$.}
        \label{fig:a}
    \end{subfigure}~
    \begin{subfigure}[]{0.5\linewidth}    
        \includegraphics[width=0.8\textwidth]{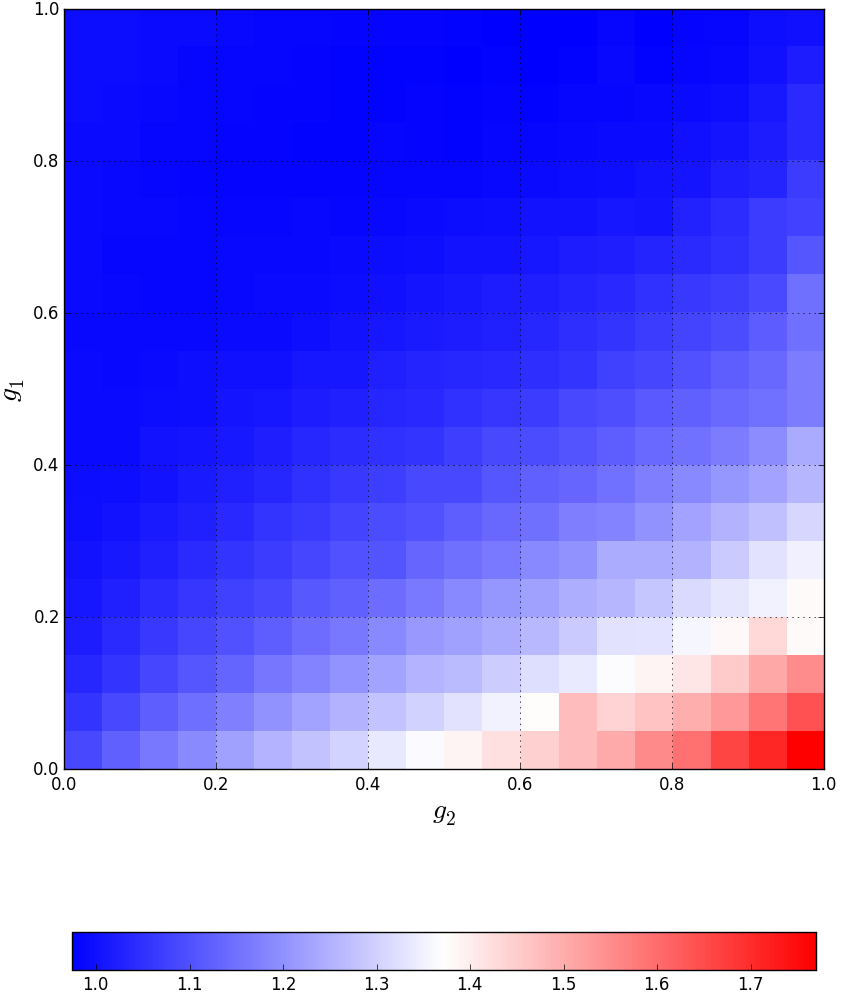}
        \caption{${U_{\g}^{*}}/{U_{\g}^{\one}}$.}
        \label{fig:b}    
    \end{subfigure} 
    \caption{Scaled reward of player A and the lower bound on achievable gain by optimizing the model coefficient for $g_{12}=g_1g_2$.}
\end{figure}

\section{Conclusion}
We introduced a game theoretic framework for studying the performance of predictive linear models in competitions. We focused on a two player competition, and demonstrated that if both players put aside the theoretical predictions of the target, and optimize their corresponding models based on the practical reward, then the player with the inferior prediction gains extra rewards. The reward can be significant depending on the players' knowledge regimes. This underlines the value of model optimization beyond the theoretical prediction. A more rigorous future study shall consider the following further directions: 1) devising an analytical study for the trade-off between cost of feature (knowledge) expansion and reward enhancement,  1) obtaining theoretical characterizations of optimal coefficients for each player, 2) studying mixed strategies and minmax equilibrium, which leads to suggestions for model combination (e.g., alpha combination in trading and investment), 3) studying the problem in a dynamic setup where players adaptively optimize their models. One can potentially characterize the convergence of the system and propose efficient learning/optimization algorithms for the latter problem. 4) Studying a multi-player game where features access has a particular distribution among a large number of players. Understanding the optimization gain of players as a function of where in the spectrum of knowledge they stand is a challenging and important practical problems that conforms to asymptotic cases of real world competitive applications.

%\bibliographystyle{IEEEtran}
%\bibliography{refs}

% Generated by IEEEtran.bst, version: 1.14 (2015/08/26)

\end{document}